\documentclass{article}
\usepackage{graphicx} 
\usepackage{amsmath,amsthm,amssymb,bbm,stmaryrd,bm}
\usepackage{url}
\usepackage{dsfont}
\usepackage{color}
\usepackage{wrapfig}
\usepackage{listing}
\usepackage{amsopn}

\newcommand{\vp}{\bm{p}}

\newcommand{\vs}{\bm{s}}               
\newcommand{\vt}{\bm{t}}               
\newcommand{\vu}{\bm{u}}               
               
\newcommand{\vw}{\bm{w}}       \newcommand{\vwh}{\hat{\bm{w}}}        
\newcommand{\vx}{\bm{x}}               
\newcommand{\vy}{\bm{y}}       \newcommand{\vyh}{\hat{\bm{y}}}        \newcommand{\yh}{\hat{y}}

\newcommand{\vepsilon}{\bm{\epsilon}}     
\newcommand{\vphi}    {\bm{\phi}}

\newcommand{\mi}{\bm{I}}

    \newcommand{\Nc}{\mathcal{N}}  
      
    \newcommand{\Pc}{\mathcal{P}}

  \newcommand{\Sc}{\mathcal{S}}

    \newcommand{\Xc}{\mathcal{X}}  
    \newcommand{\Yc}{\mathcal{Y}}

\newcommand{\R}{\mathbb{R}}

\newcommand{\E}{\mathbb{E}}
\renewcommand{\P}{\mathbb{P}}
\newcommand{\I}{\mathbb{I}}
\newcommand{\reals}{\mathbb{R}}

\DeclareMathOperator*{\argmin}{argmin}
\DeclareMathOperator*{\argmax}{argmax}

\usepackage{amsthm}
\newtheorem{theorem}{Theorem}
\newtheorem{lemma}[theorem]{Lemma}

\newcommand{\eq}[1]{(\protect\ref{#1})}

\usepackage{hyperref}
\usepackage{url}
\usepackage[authoryear,sort&compress]{natbib}

\usepackage{algorithm}
\usepackage{algorithmic}

\usepackage{subfigure} 

\begin{document}

%

%
\title{Risk Minimization in Structured Prediction using Orbit Loss}
\author{
    Danny Karmon\\
    Dept. of Computer Science\\
    Bar-Ilan University, Israel\\
    danny.karmon@biu.ac.il
  \and
    Joseph Keshet\\
    Dept. of Computer Science\\
    Bar-Ilan University, Israel\\
    joseph.keshet@biu.ac.il
}
\maketitle

\begin{abstract}
We introduce a new surrogate loss function called orbit loss in the structured prediction framework, which has good theoretical and practical advantages. While the orbit loss is not convex, it has a simple analytical gradient and a simple perceptron-like learning rule. We analyze the new loss theoretically and state a PAC-Bayesian generalization bound. We also prove that the new loss is consistent in the strong sense; namely, the risk achieved by the set of the trained parameters approaches the infimum risk achievable by any linear decoder over the given features. Methods that are aimed at risk minimization, such as the structured ramp loss, the structured probit loss and the direct loss minimization require at least two inference operations per training iteration. In this sense, the orbit loss is more efficient as it requires only one inference operation per training iteration, while yields similar performance. We conclude the paper with an empirical comparison of the proposed loss function to the structured hinge loss, the structured ramp loss, the structured probit loss and the direct loss minimization method on several benchmark datasets and tasks.
\end{abstract}


\section{Introduction}

There are three main differences between binary classification problems and structured prediction problems. First, the input to a binary classifier is a feature vector of a fixed length and the output is restricted to two possible labels, whereas in structured prediction both the input and the output are structured objects (a graph, an acoustic speech utterance, a sequence of words, an image). Second, the structured output space is potentially exponentially large (all possible phoneme or word sequences, all possible taxonomy graphs, all possible human poses, etc.). And third, while in binary classification the system's performance is evaluated using the error rate, i.e., 0-1 loss, in structured prediction each task often has its own evaluation metric or cost, such as word error rate in speech recognition, the BLEU score in machine translation, the NDCG score in information retrieval, or the intersection-over-union score in visual object segmentation. Some of these are involved functions, which are non-decompostable in the output space.


There is significant literature on learning parameters for structured prediction and graphical models. Ultimately, the goal in learning is to find the model parameters so as to minimize the expected cost, or risk, where the expectation is taken with respect to a random draw of input-output pairs from a fixed but unknown distribution. Since the expectation cannot be evaluated because the underlying probability is unknown, and since the cost is often a non-convex combinatorial function (which is hard to minimize directly), the learning problem is formulated as an optimization problem where the parameters are found by minimizing a trade-off between a measure of the goodness of fit (loss) to the training data and a regularization term. In discriminative training, the loss function should be directly related to the cost between the model prediction and the target label, averaged over the training set. 

The most common approaches to structured prediction, namely structured perceptron, structural support vector machine (SVM) and conditional random fields (CRF), do not directly minimize the risk. The structured perceptron \citep{Collins02} solves a feasibility problem, which is independent of the cost. In structural SVM \citep{TsochantHoJoAl05} the measure of goodness is a convex upper bound to the cost called structural hinge loss. It is based on a generalization of the binary SVM hinge loss to the structured case, and there is no guarantee for the risk. While there exists generalization bounds for the structured hinge loss \citep[e.g., ][]{McAllester06,TaskarGuKo03}, they all include terms which are not directly related to the cost, such as the Hamming loss, and inherently the structured hinge loss cannot be consistent as it fails to converge to the performance of the optimal linear predictor in the limit of infinite training data \citep{McAllester06}. In CRFs the measure of goodness is the log loss function, which is independent of the cost \citep{LaffertyMcPer01}. \cite{SmithEi06} tried to address this shortcoming of CRFs and proposed to minimize the risk under the Gibbs measure. While it seems that this loss function is consistent, we are not aware of any formal analysis.

Recently, several works have focused on directly minimizing the expected cost. In particular, \citet{McAllesterHaKe10} presented a theorem stating that a certain perceptron-like learning rule, involving feature vectors derived from cost-augmented inference, directly corresponds to the gradient of the risk. Direct loss needs two inference operations per training iteration and is extremely sensitive to its hyper-parameter. \citet{DoLeTeChSm08} generalized the notion of the ramp loss from binary classification to structured prediction and proposed a loss function, which is a non-convex bound to the cost, and was found to be a tighter bound than the structured hinge loss function. The structured ramp loss also needs two inference operations per training iteration. \citet{KeshetMcHa11} generalized the notion of the binary probit loss to the structured prediction case. The gradient of this non-convex loss function can be approximated by averaging over samples from the unit-variance isotropic normal distribution, where for each sample an inference with a perturbed weight vector is computed. In order to gain stability in the gradient computation, hundreds to thousands of inference operations are required per training iteration, hence the update rule is computationally-heavy.

The goal of this work is to propose a new learning update rule for structured prediction which results in fast training on one hand and aims at minimizing the risk on the other hand. We define a new loss function, called \emph{orbit}, where its gradient has a close simple analytical form, which is very close to the structured perceptron update rule. We state a finite sample generalization bound for this loss function and show that it is consistent in the strong sense. That is, for any feature map (finite or infinite dimensional) the loss function yields predictors approaching the infimum risk achievable by any linear predictor over the given features. The update rule of this new loss involves one inference operation per training iteration, similar to the structured perceptron or the structural SVM, and hence faster (per training iteration) than ramp, probit and direct loss minimization. In a series of experiments we showed that the new loss function performs similar to other approaches that were designed to minimize the risk. 

The paper is organized as follows. In Section~\ref{sec:problem_setting} we state the problem formally. In Section~\ref{sec:orbit} we introduce the new surrogate loss function and its update rule. In Section~\ref{sec:analysis} we present the analysis for our new methods, including proofs for both consistency and generalization bound. In Section~\ref{sec:experiments} we present a set of experiments and compare the new learning rule to other algorithms. We conclude the paper in Section~\ref{sec:conclusion}.


\section{Formal Settings}\label{sec:problem_setting}

We formulate the structured supervised learning problem by setting $\Xc$ to be an abstract set of all possible input objects and $\Yc$ to be an abstract set of all possible output targets. We assume that the input objects $\vx\in\Xc$ and the target labels $\vy \in \Yc$ are drawn from an unknown joint distribution $\rho$. We define a set of $d$ fixed mappings $\vphi: \Xc \times \Yc \rightarrow \reals^d$ called \emph{feature functions} from the set of input objects and target labels to a real vector of length $d$.

Here we consider a linear decoder with parameters $\vw \in \reals^d$, such that the parameters weight the feature functions. We denote the score of label $\vy$ by $\vw\cdot\vphi(\vx,\vy)$, given the input $\vx$. The decoder predicts the label $\vyh_{\vw}$ with the highest score:
\begin{equation}
\label{eq:yw}
\vyh_{\vw}(\vx) = \argmax_{\vy \in \Yc} ~ \vw\cdot \vphi(\vx, \vy)
\end{equation} 
Ideally, we would like to find the parameters $\vw$ that optimize the risk for unseen data. Formally, we define the \emph{cost} function, $\ell(\vy, \vyh_{\vw})$, to be a non-negative measure of error when predicting $\vyh_{\vw}$ instead of $\vy$ as the label of $\vx$. We assume that $\ell(\vy,\vy)=0$ for all $\vy$. Often the desired evaluation metric is a utility function that needs to be maximized (like BLEU or NDCG) and then we define the cost to be 1 minus the evaluation metric.  

Our goal is to minimize the risk:
\begin{equation}
\label{eq:w*}
\vw^* = \argmin_{\vw} ~ \mathbb{E}_{(\vx,\vy) \sim \rho} [\ell(\vy,\vyh_{\vw}(\vx))]. 
\end{equation}
Since the distribution $\rho$ is unknown, we use a training set $\Sc=\{(\vx_i,\vy_i)\}_{i=1}^m$ of $m$ examples that are drawn i.i.d. from $\rho$, and replace the expectation in \eq{eq:w*} with a mean over the training set and a regularization factor $\frac{1}{2}\|\vw\|^2$. The cost is often a combinatorial non-convex quantity, which is hard to minimize, hence it is replaced with a \emph{surrogate loss}, denoted $\bar{\ell}(\vw,\vx,\vy)$. Different algorithms use different surrogate loss functions. Overall the objective function of \eq{eq:w*} transforms into the following objective function
\begin{equation}
\label{eq:reg-loss}
\vw^* = \argmin_{\vw}  ~ \frac{1}{m}\sum_{i=1}^{m} \bar{\ell}(\vw,\vx_i,\vy_i) + \frac{\lambda}{2} \|\vw\|^2, 
\end{equation} 
where $\lambda$ is a trade-off parameter between the loss term and the regularization factor.


\section{Orbit Loss}\label{sec:orbit}

Denote by $\Delta\vphi(\vy,\vy')$ the difference between the feature functions of the labels $\vy, \vy' \in \Yc$, respectively: 
$$
\Delta\vphi(\vy,\vy')=\vphi(\vx,\vy)-\vphi(\vx,\vy').
$$
Define by $\delta\vphi(\vy,\vy')$ the normalized version of $\Delta\vphi(\vy,\vy')$ as follows:
\begin{equation}
\delta\vphi(\vy,\vy') = 
\left\{ 
\begin{array}{ll}
 {\Delta\vphi(\vy,\vy')}/{\|\Delta\vphi(\vy,\vy')\|} & \textrm{if $\vy \ne \vy'$}\\
 \boldsymbol{0} & \textrm{if $\vy = \vy'$}
\end{array} \right. .
\end{equation}

The \emph{orbit} surrogate loss function is formally defined as follows:
\begin{equation}\label{eq:orbit_loss_def}
\bar{\ell}_{orbit}(\vw,\vx,\vy)= 
\P_{\varepsilon\sim\Nc(0, 1)}
\Big[\varepsilon > \vw\cdot\delta\vphi(\vy,\vyh_{\vw})\Big] \, \ell(\vy, \vyh_{\vw}).
\end{equation}
That is, the orbit loss is equal to the cost multiplied by the probability that the prediction score $\vw\cdot\vphi(\vx,\vyh_{\vw})$ plus a small number $\varepsilon$ is greater than the score of the target label $\vw\cdot\vphi(\vx,\vy)$.

We now derive the gradient-based learning rule for this loss function, which helps to describe some of its properties. The loss has a simple analytical gradient:
\begin{align}
\nabla_{\vw} \Big[
\P_{\varepsilon\sim\Nc(0, 1)}
	\left[ \varepsilon > \vw\cdot\delta\vphi(\vy,\vyh)\right] \, \ell(\vy, \vyh) \Big] 
&=	\nabla_{\vw} \left[ 
\frac{1}{\sqrt{2\pi}} \int_{\vw\cdot\delta\vphi(\vy,\vyh)}^{\infty} e^{{-z^2}/{2}} dz \, \ell(\vy,\vyh)  
\right] \\ 
&=  -\frac{1}{\sqrt{2\pi}}  e^{{-|\vw\cdot\delta\vphi(\vy,\vyh)|^2}/{2}} \ell(\vy,\vyh) \, \delta\vphi(\vy,\vyh).
\end{align}
The update rule of the orbit loss is the following:
\begin{equation}\label{eq:orbit_update_rule}
	\vw ~\gets~ (1-\eta\lambda)\,\vw  ~+~ \eta \, e^{{-|\vw\cdot\delta\vphi(\vy,\vyh_{\vw})|^2}/{2}} \, \ell(\vy,\vyh_{\vw}) \, \delta\vphi(\vy,\vyh_{\vw}).
\end{equation}

Note that when the prediction label $\vyh_{\vw}$ is close to the target label $\vy$ in terms of the decoding score, that is, when the term $\vw\cdot\delta\vphi(\vyh_{\vw},\vy)$ is relatively small, the exponent is close to 1. Under this condition the update rule becomes
\begin{equation}\label{eq:orbit_update_rule2}
	\vw \gets (1-\eta\lambda)\,\vw ~+~ \eta \, \ell(\vy,\vyh) \, \delta\vphi(\vy,\vyh_{\vw}),
\end{equation}
which generalizes the regularized structured perceptron's update rule \citep{Collins02,zhang2014regularized}. Namely
	\begin{equation}\label{eq:perceptron_update_rule}
		\vw \gets (1-\eta\lambda)\,\vw + \eta \, \mathbbm{1}\{\vy \neq \vyh_{\vw}\} \, \delta\vphi(\vy,\vyh_{\vw}),
	\end{equation} 
where $\mathbbm{1}\{\pi\}$ is an indicator function, equals 1 if the predicate $\pi$ holds and equals 0 otherwise. 

A nice property of this update rule is that the cost function does not need to be decomposable in the size of the output. Decomposable cost functions are needed in order to solve the cost-augmented inference that is used in the training of structural SVMs  \citep{TsochantHoJoAl05,ranjbar2013optimizing}, direct loss minimization  \citep{McAllesterHaKe10}, or structured ramp loss \citep{DoLeTeChSm08}. It means that cost functions like word error rate or intersection-over-union can be used without being approximated. 

Another property of the orbit loss is its similarity to the structured probit loss \citep{KeshetMcHa11}. The probit loss was derived from the concept of stochastic decoder in the PAC-Bayesian framework \citep{McAllester98,McAllester03} and was shown to have both good theoretical properties and practical advantages \citep{KeshetMcHa11}. The structured probit loss is defined as follows
\begin{equation}
\bar{\ell}_{probit}(\vw,\vx,\vy) = \E_{\vepsilon\sim\Nc(\boldsymbol{0}, \mi)} \left[ \ell(\vy, \vyh_{\vw+\vepsilon}) \right],
\end{equation}
where $\vepsilon \in \R^d$ is a $d$-dimensional isotropic Normal random vector. Note that the orbit loss \eqref{eq:orbit_loss_def} can be written as follows:
\begin{multline}
\P_{\varepsilon\sim\Nc(0, 1)}
\left[\varepsilon > \vw\cdot\delta\vphi(\vy,\vyh_{\vw})\right] \, \ell(\vy, \vyh_{\vw}) 
\\ 
=  \P_{\vepsilon\sim\Nc(\boldsymbol{0}, \mi)}
\left[-\vepsilon\cdot\delta\vphi(\vy,\vyh_{\vw}) > \vw\cdot\delta\vphi(\vy,\vyh_{\vw})\right] \, \ell(\vy, \vyh_{\vw}). 
\end{multline}
The last equation holds since the inner product of an isotropic normal random vector $\vepsilon$ with a unit-norm vector $\delta\vphi(\vy,\vyh)$ is a zero-mean unit variance normal random variable. Writing the probability as the expectation of an indicator function, we have  
\begin{equation}
\bar{\ell}_{orbit}(\vw,\vx,\vy) 
= \E_{\vepsilon\sim\Nc(\boldsymbol{0}, \mi)}  \Big[ \I\left\{
(\vw+\vepsilon)\cdot\delta\vphi(\vy,\vyh_{\vw})\! <\! 0\right\} \Big] \,\ell(\vy, \vyh_{\vw}).
\end{equation}
Assuming $\vyh_{\vw+\vepsilon}=\vyh_{\vw}$ for a value $r$ small enough that $\vepsilon \in B(\boldsymbol{0},r)$, where $B(\boldsymbol{0},r)$ is a ball of radius $r$ centered at $\boldsymbol{0}$, we can bring the cost function into the expectation term, that is
\begin{equation} \label{eq:expectation_to_orbit}
		\E_{\vepsilon\sim\Nc(\boldsymbol{0}, \mi)} \left[ \I\{(\vw+\vepsilon)\cdot\delta\vphi(\vy,\vyh_{\vw+\vepsilon})\! <\! 0\}\,\ell(\vy, \vyh_{\vw+\vepsilon}) \right]  = 
		\E_{\vepsilon\sim\Nc(\boldsymbol{0}, \mi)} \left[ \ell(\vy, \vyh_{\vw+\vepsilon}) \right],
\end{equation}
which is the structured probit loss.


\section{Analysis}\label{sec:analysis}

In this section we analyze the orbit loss. We derive a generalization bound based on the PAC-Bayesian theory, where we start by upper-bounding the probit loss with the orbit loss and then plugging it into a PAC-Bayesian generalization bound. Then we show that the decoder's parameters, which are estimated by optimizing the regularized orbit loss in the limit of infinite data, approach the infimum risk achievable by any linear decoder.

Recall that the structured probit loss is defined as: 
\begin{equation}
\bar{\ell}_{probit}(\vw,\vx,\vy) = \E_{\vepsilon\sim\Nc(\boldsymbol{0}, \mi)} \left[ \ell(\vy, \vyh_{\vw+\vepsilon}) \right].
\end{equation}
The following theorem states a generalization bound for the probit loss function \citep{KeshetMcHa11}.
\begin{theorem}[Generalization of probit loss]\label{thm:probit_generalization}
For a fixed $\gamma>1/2$ we know that, with a probability of at least $1-\delta$ over the draw of the training data, the following holds simultaneously for all $\vw$:
\begin{equation}
\E_{(\vx,\vy)\sim\rho} \left[ \bar{\ell}_{probit}(\vw,\vx,\vy) \right]
  \le \frac{1}{1-\frac{1}{2\gamma}} \Bigg( 
\frac{1}{m}\sum_{i=1}^{m} \bar{\ell}_{probit}(\vw,\vx_i,\vy_i) + \frac{\gamma}{2m}\|\vw\|^2  + \frac{\gamma}{m}\ln \frac{1}{\delta} \Bigg).
\end{equation}
\end{theorem}
Later this generalization bound will help us  state a similar bound for the orbit loss.

We now analyze the orbit loss. 
Let $\eta$ be the minimal distance between the score of the predicted label $\vyh$ to the score of its closest \emph{different} label $\vy'$ by a constant $\eta$:
\begin{equation}\label{eq:condition}
\min_{\vy' \ne  \vyh} ~ \vw\cdot\delta\vphi(\vyh, \vy') \ge \eta
\end{equation}
for $\vyh \ne  \vy'$.
The following lemma upper bounds the probit loss with the orbit loss. 
\begin{lemma}\label{lem:l_probit_l_orbit}
For a finite $\sigma > 0$ and a cost function $\ell(\vy,\vy') \in [0,1]$ for all $\vy$, $\vy'$, the following holds:
	\begin{equation}
		\bar{\ell}_{probit}(\vw/\sigma, \vx, \vy) \leq \bar{\ell}_{orbit}(\vw/\sigma, \vx, \vy) + \sigma,
	\end{equation}
for $\eta \ge \sigma\sqrt{2\ln \frac{m}{\sigma}}$.
\end{lemma}
For the brevity of the explanation we call $\vyh_{\vw}$ the \emph{predicted label} and we call $\vyh_{\vw+\vepsilon}$ the \emph{perturbed label}. The idea behind the proof is to split the structured probit loss cases of $\vepsilon$ for which the predicted label and the perturbed label are the same, and the case in which they differ. We show that the probability of the labels being equal is upper-bounded by the orbit loss, and the probability of the labels being different is upper-bounded by an exponential term that approaches zero when the norm of $\vepsilon$ approaches zero.
\begin{proof}
From the law of total expectation we have
\begin{equation}\label{eq:probit_loss_split}
\E_{\vepsilon \sim \mathcal{N}(\boldsymbol{0},\mi)}\Big[\ell(\vyh_{\vw+\vepsilon}, \vy)\Big] 
\leq ~~ \E_{\vepsilon}\Big[\mathbbm{1} \{ \vyh_{\vw+\vepsilon} = \vyh_{\vw} \} \ell(\vyh_{\vw+\vepsilon}, \vy)\Big] 
 +  \P_{\vepsilon} \Big[ \vyh_{\vw+\vepsilon} \neq \vyh_{\vw} \Big],
\end{equation}
where we upper bound the cost by 1 in the second term. 

First, let us focus on the first term of the inequality. For this term $\vyh_{\vw+\vepsilon} = \vyh_{\vw}$, which means that 
\begin{equation}
\E_{\vepsilon}\Big[ \mathbbm{1} \{ \vyh_{\vw+\vepsilon} = \vyh_{\vw} \}
\ell(\vyh_{\vw+\vepsilon}, \vy)\Big]   
 =\P_{\vepsilon} \Big[ \vyh_{\vw+\vepsilon} = \vyh_{\vw} \Big] \, \ell(\vyh_{\vw}, \vy). 
\end{equation}
By definition of the inference rule \eqref{eq:yw} for any vector $\vu$, we have $\vu\cdot\delta\vphi(\vyh_{\vu}, \vy) \ge 0$ for all $\vy$. Therefore the probability that $\vyh_{\vw+\vepsilon} = \vyh_{\vw}$ can be expressed as follows: 
\begin{equation} 
\P_{\vepsilon} \Big[ \vyh_{\vw+\vepsilon} = \vyh_{\vw} \Big]  
=\P_{\vepsilon} \Big[ (\vw+\vepsilon)\cdot\delta\vphi(\vyh_{\vw+\vepsilon}, \vyh_{\vw}) \le 0 \Big] 
\end{equation}
which, in turn, can be expressed as
\begin{equation}
 \P_{\vepsilon} \Big[ \vw\cdot\delta\vphi(\vyh_{\vw+\vepsilon}, \vyh_{\vw}) \le -\vepsilon \cdot\delta\vphi(\vyh_{\vw+\vepsilon}, \vyh_{\vw}) \Big]
  \le  \P_{\vepsilon} \Big[ \vw\cdot\delta\vphi(\vy, \vyh_{\vw}) \le -\vepsilon \cdot\delta\vphi(\vyh_{\vw+\vepsilon}, \vyh_{\vw})  \Big] ,
\end{equation}
where replacing $\vyh_{\vw+\vepsilon}$ with $\vy$ increases the event size, thereby increasing the probability. Replacing the inner product of an isotropic normal random vector $\vepsilon$ with a unit-norm vector $\delta\vphi(\vy,\vyh_{\vw})$ with a zero-mean unit variance normal random variable, we get:
\begin{equation}
\P_{\vepsilon} \Big[ \vw\cdot\delta\vphi(\vy, \vyh_{\vw}) \le -\vepsilon \cdot\delta\vphi(\vyh_{\vw+\vepsilon}, \vyh_{\vw})  \Big] 
= \P_{\varepsilon\sim \mathcal{N}(0,1)}\Big[ \varepsilon > \vw\cdot\delta\vphi(\vy,\vyh_{\vw}) \Big] .
\end{equation}

The second term of the left-hand side of \eqref{eq:probit_loss_split} can be expressed as follows:
$$
\P_{\vepsilon} [ \vyh_{\vw+\vepsilon} \neq \vyh_{\vw} ] = \P_{\vepsilon} [ (\vw+\vepsilon)\cdot\delta\vphi(\vyh_{\vw+\vepsilon}, \vyh_{\vw}) > 0 ].
$$
We have 
\begin{align}
\P_{\vepsilon} [ (\vw+\vepsilon)\cdot\delta\vphi(\vyh_{\vw+\vepsilon}, \vyh_{\vw}) > 0 ]
&= \P_{\vepsilon} [ \vepsilon\cdot\delta\vphi(\vyh_{\vw+\vepsilon}, \vyh_{\vw}) > \vw\cdot\delta\vphi(\vyh_{\vw}, \vyh_{\vw+\vepsilon}) ]\\ 
&\le \P_{\vepsilon} [ \vepsilon\cdot\delta\vphi(\vyh_{\vw+\vepsilon}, \vyh_{\vw}) > \eta ]. 
\end{align}
We finalized the proof by bounding the last equation for a $\sigma$-scaled version of $\vepsilon$,
\begin{equation}
\P_{\vepsilon\sim\mathcal{N}(\boldsymbol{0},\mi)} [ \sigma\vepsilon\cdot\delta\vphi(\vyh_{\vw+\sigma\vepsilon}, \vyh_{\vw}) > \eta ]  
=\P_{\varepsilon\sim\mathcal{N}(0,1)} [ \sigma\varepsilon > \eta ]
\le \exp \left( -\frac{\eta^2}{2\sigma^2} \right) = \frac{\sigma}{m},
\end{equation}
where the first equation holds since the inner product of an isotropic normal random vector $\vepsilon$ with a unit-norm vector $\delta\vphi(\vyh_{\vw+\sigma\vepsilon}, \vyh_{\vw})$ is a zero-mean unit variance normal random variable; and the second equation holds for $\eta \ge \sigma\sqrt{2\ln \frac{m}{\sigma}}$. Using the union bound over the draw of a sample of size $m$ concludes the proof.
\end{proof}

Plugging Lemma~\ref{lem:l_probit_l_orbit} into the bound of Theorem~\ref{thm:probit_generalization}, we get the following generalization bound for the orbit loss.
\begin{theorem}[Generalization of orbit loss]\label{thm:generalization_orbit}
For a fixed $\gamma > 1/2$ and assuming \eqref{eq:condition} holds with $\eta \ge \sigma\sqrt{2 \ln (m/\sigma)}$, we know that with a probability of at least $1-\delta$ over the draw of the training data the following holds true simultaneously for all $\vw$ and for all $\sigma > 0$:
\begin{multline}\label{eq:generalization_orbit}
\E_{(\vx,\vy)\sim\rho} \left[ \bar{\ell}_{probit}(\vw/\sigma,\vx,\vy) \right]
  \\ \le \frac{1}{1-\frac{1}{2\gamma}} \Bigg( 
\frac{1}{m}\sum_{i=1}^{m} \bar{\ell}_{orbit}(\vw/\sigma,\vx_i,\vy_i)  
+ \frac{\gamma}{2m\sigma^2}\|\vw\|^2  + \sigma+ \frac{\gamma}{m}\ln \frac{1}{\delta} \Bigg).
\end{multline}
\end{theorem}

We will now prove that the orbit loss is consistent. We start with the observation that when the norm of the weight vector $\vw$ goes to infinity, the orbit loss approaches the cost:
\begin{lemma}\label{lemma:orbit_infty}
\begin{equation}
\lim_{\alpha\to \infty}  \bar{\ell}_{orbit}(\alpha\vw,\vx,\vy) =  \ell(\vy,\vyh_{\vw}),
\end{equation}
assuming that $\ell(\vy,\vy)=0$ for all $\vy$.
\end{lemma}
\begin{proof}
Recall that $\vyh_{\vw} = \arg \min_{\vy'} \vw\cdot\delta\vphi(\vy,\vy')$ therefore $\vw\cdot\delta\vphi(\vy, \vyh_{\vw}) \le 0$. Also note that scaling the parameters $\vw$ does not change the prediction, $\vyh_{\alpha\vw} = \vyh_{\vw}$. We have:
\begin{multline}
\lim_{\alpha\to \infty} \P_{\varepsilon\sim\Nc(0,1)}[\varepsilon > \alpha\vw\cdot\delta\vphi(\vy, \vyh_{\vw})] \,  \ell(\vy,\vyh_{\vw}) \\ = \P_{\varepsilon\sim\Nc(0,1)}[\varepsilon > -\infty] \,  \ell(\vy,\vyh_{\vw}) = \ell(\vy,\vyh_{\vw}).
\end{multline}
\end{proof}

Consider the following training objective:	
\begin{equation}\label{eq:orbit_objective}
	\vwh_m = \arg\min_{\vw} ~ \frac{1}{m} \sum_{i=1}^{m} \bar{\ell}_{orbit}(\vw,\vx_i,\vy_i) + \frac{\lambda_m}{2m}\|\vw\|^2.
\end{equation}

\begin{theorem}[Consistency of orbit loss]\label{theorem:consistency_orbit}

For $\hat{\vw}_m$ defined by \eqref{eq:orbit_objective}, if the sequence $\lambda_m/\ln^2 m$ increases without bound, and the sequence $\lambda_m/(m\ln m)$ converges to zero, then with a probability of one over the draw of the infinite sample we have: 
\begin{align}
&\lim_{m \rightarrow \infty} \E_{(\vx,\vy) \sim \rho}
\Big[ \bar{\ell}_{probit}((\ln m) \hat{\vw}_m, \vx, \vy) \Big]
\\ \nonumber
& ~~~~~~~ = \inf_{\vw} \E_{(\vx,\vy) \sim \rho}
\Big[ \ell(\vy,\vyh_{\vw}(\vx)) \Big].
\end{align}

\end{theorem} 
\begin{proof}
Set $\delta=1/m^2$, $\sigma = 1/\ln m$, and $\gamma_m = \lambda_m/\ln^2 m$ into the bound \eq{eq:generalization_orbit}. We decompose $\vw$ into a scalar $\alpha$, corresponding to the norm of $\vw$, and a unit norm vector $\vw^*$. Last, using Chernoff we upper-bound 
$\frac{1}{m}\sum_{i=1}^{m} \bar{\ell}_{orbit}(\vw,\vx_i,\vy_i)$ by 
$\E_{(\vx,\vy)\sim\rho} \left[ \bar{\ell}_{orbit}(\vw,\vx,\vy) \right] + \sqrt{\ln m /m}$ to get
\begin{align}\label{eq:generalization_bound_orbit_loss_plugged}
\E_{(\vx,\vy)\sim\rho} \left[ \bar{\ell}_{probit}((\ln m)\vw_m,\vx,\vy) \right]
&\leq\E_{(\vx,\vy)\sim\rho} \left[ \bar{\ell}_{probit}((\ln m)\alpha\vw^*,\vx,\vy) \right] \\ 
&\leq \frac{1}{1-\frac{\ln^2 m}{2\lambda_m}}\Bigg( 
	\E_{(\vx,\vy)\sim\rho} \Big[ \bar{\ell}_{orbit}(\alpha\vw^*,\vx,\vy) \Big] \\ \nonumber
&~~~~~~~~~~~~~~~~~~~~ + \sqrt{\frac{\ln m}{m}} + \frac{\lambda_m \alpha^2}{2m} +  \frac{1}{\ln m} 
	 + \frac{2\lambda_m}{m \ln m}\Bigg),
\end{align}
where $\vw_m$ is the minimizer of the right-hand side of the bound in \eqref{eq:generalization_orbit}, as well as of the optimization problem \eqref{eq:orbit_objective}. Taking the limit when the number of examples $m$ approaches infinity on both sides we have
\begin{equation}
\lim_{m \to \infty} \E_{(\vx,\vy)\sim\rho} \left[ \bar{\ell}_{probit}((\ln m)\vw_m,\vx,\vy) \right] \le  \E_{(\vx,\vy)\sim\rho} \Big[ \bar{\ell}_{orbit}(\alpha\vw^*,\vx,\vy) \Big] 
\end{equation}
Noting that by 
\begin{equation}
\E_{(\vx,\vy)\sim\rho} \left[ \bar{\ell}_{probit}(\vw,\vx,\vy) \right] \ge \inf_{\vw} \E_{(\vx,\vy)\sim\rho} \left[ \ell(\vy,\vy_{\vw}) \right],
\end{equation}
and letting $\alpha$ approach infinity using Lemma~\ref{lemma:orbit_infty} concludes the proof.
\end{proof}


\section{Experiments}\label{sec:experiments}

We evaluated the performance of the orbit loss by executing a number of experiments on several domains and tasks and compared the results with other approaches that are aimed at risk minimization, namely direct loss minimization \citep{McAllesterHaKe10}, structured ramp loss \citep{DoLeTeChSm08}, and structured probit loss \citep{KeshetMcHa11}. For a reference, we present results for the structured perceptron, as we wanted to stress the empirical differences between the update rule in \eqref{eq:orbit_update_rule2} and the one in \eqref{eq:perceptron_update_rule}, as well as for the structured hinge loss.

\subsection{MNIST}

In our first experiment we tested the orbit update rule on a multiclass problem. MNIST is a dataset of handwritten digit images (10 classes). It is divided into a training set of 50,000 examples, a test set of 10,000 examples and a validation set of 10,000 examples. We preprocess the data by normalizing the input images, and reducing the dimension from the original 784 attributes to 100 using PCA.

We used the orbit update rule as in \eqref{eq:orbit_update_rule}. We defined the weight vector $\vw$ as a concatenation of 10 weight vectors $\vw=(\vw^0, \vw^1, \ldots, \vw^9)$, each corresponding to one of the 10 digits. The update rule of example $(\vx_i, y_i)$, $\vx_i\in \R^{100}$, $y_i\in\{0,\ldots,9\}$ can be simplified based on Kesler's construction \citep{CrammerSi01a} as follows:
\begin{align*}
\vw^{y_i} &\gets (1-\eta\lambda)\,\vw^{y_i} + \eta \, e^{-|\vw^{y_i}\cdot\vx_i - \vw^{\yh}\cdot\vx_i|^2/2}\, \ell(\yh,y_i) \,\, \vx_i \\
\vw^{\yh}~ &\gets (1-\eta\lambda)\,\vw^{\yh} \,\, - \eta \, e^{-|\vw^{y_i}\cdot\vx_i - \vw^{\yh}\cdot\vx_i|^2/2}\, \ell(\yh,y_i) \, \vx_i \\
\vw^r~ &\gets (1-\eta\lambda)\,\vw^r ~~ ~~~~~~~~~~~~~~ ~~~ ~~~~~\mbox{for all $r\ne y_i, \yh$}
\end{align*}
Note that the exponent values throughout the training were very close to 1 and, practically, the update rule \eqref{eq:orbit_update_rule2} could be used. 

To properly evaluate the orbit loss we ran the experiment with two different cost functions for $\ell(\yh, y)$: 0-1 loss and a semi-randomized matrix. We did so because the update rule \eqref{eq:orbit_update_rule2} is identical to the structured perceptron update rule under the 0-1 loss. 

In the first case, we set $\eta=\eta_0/\sqrt{t}$, for $\eta_0=0.1$, $t$ is the iteration number, and $\lambda=0.001$. We also trained a multiclass perceptron with $\eta=1$, and $\lambda=0$, and a multiclass SVM with $C=0.01$ \citep{CrammerSi01a}. All of the hyper-parameters were chosen on the validation set. In all of the experiments we ran 4 epochs over the training data and used a linear kernel. 

The results are given in Table~\ref{tab:mnist1} and suggest that there is a slight advantage for the orbit loss over the other algorithms. Recall that we previously showed that the perceptron is a special case of orbit loss and under this setting, in which $\ell(\yh, y)$ = 0-1 loss, hence the only difference between the results in the table is due to the regularization factor used with the orbit loss.

\begin{table}[h]
\renewcommand{\arraystretch}{1.4}
\caption{Error rate of MNIST trained and evaluated with 0-1 loss.}\label{tab:mnist1}
\vspace{0.15in}
\centering
\begin{small}
\begin{tabular}{lc}
\hline
\multicolumn{1}{c}{Algorithm} &\multicolumn{1}{c}{Error rate (0-1 loss)}
\\ \hline 
Multiclass perceptron & 8.54\% \\
Multiclass SVM &8.51\% \\
Orbit &\textbf{8.34\%} \\
\hline
\end{tabular}
\end{small}
\end{table}

As mentioned above, this experiment was executed once again, setting the cost function, $\ell(\yh, y)$, to be a semi-randomized matrix. We generated a randomized cost matrix of size 10 $\times$ 10, such that the elements on the diagonal were all 0, and the rest of the elements were chosen uniformly at random to be either 1 or 2. We trained multiclass perceptron, multiclass SVM, and orbit using the following hinge loss for the cost function:
\begin{equation}
\bar{\ell}_{hinge}(\vw,\vx,y) = \max_{\yh} \Big[
\ell(y,\yh) - \vw^y\cdot\vx + \vw^{\yh}\cdot\vx \Big]
\end{equation}

To ensure reliability, we ran the second experiment for each algorithm with 10 different sampled matrices and averaged the results. The results are presented in Table~\ref{tab:mnist2}. The results show a clear advantage for the orbit loss update rule in regards to the task loss. The reason is that the orbit loss can take advantage of minimizing a non 0-1 loss, as compared to perceptron.

\begin{table}[h]
\renewcommand{\arraystretch}{1.4}
\caption{Error rate of MNIST trained and evaluated using the randomized evaluation metric. The orbit and SVM were trained using the randomized evaluation metric. The reported results are average over 10 randomized evaluation metrics.}\label{tab:mnist2}
\vspace{0.15in}
\centering
\begin{small}
\begin{tabular}{lcc}
\hline
\multicolumn{1}{c}{Algorithm} &\multicolumn{1}{c}{Error rate (0-1 loss)} &\multicolumn{1}{c}{Cost }
\\ \hline
Multiclass perceptron & 8.56\% &13.28\%\\
Multiclass SVM &8.52\% & 12.84\% \\
Orbit &\textbf{8.31\%} &\textbf{11.13\%}\\
\hline
\end{tabular}
\end{small}
\end{table}

\begin{table*}[ht!]
\renewcommand{\arraystretch}{1.2}
\caption{Percentage of correctly positioned phoneme boundaries, given a predefined tolerance on the TIMIT core test set (192 examples). The first two lines (marked with *) present algorithms that were trained on a training set of 1796 examples and a validation set of 400 examples. The rest of the lines correspond to a training set of 150 examples and a validation set of 100 examples.}\label{tab:alignment}
\vspace{0.15in}
\centering
\begin{small}
\begin{tabular}{lccccc}
\hline
{} & \multicolumn{4}{c}{ $\tau$-alignment accuracy [\%]} & $\tau$-insensitive\!\!\\
\cline{2-5}
{} & $t\le10$ms & $t\le20$ms & $t\le30$ms & $t\le40$ms & loss \\[0.1cm]
\hline 
\citet{BrugnaraFaOm93}* & 79.7 & 92.1 & 96.2 & 98.1 & -\\
\citet{KeshetShSiCh07}* & {75.3} & {88.9} & {94.4} & {97.1} & -\\
Structural SVM & {79.4} & 90.5& 95.3 & 97.3 & 0.45 \\
Structured ramp-loss & 72.7 & 85.4 & 93.9 & 95.9 & 0.53 \\
Direct loss minimization & 81.0 & 90.6 & 94.5 & 96.8 & 0.47 \\
Orbit & \bf{84.4} & \bf{93.5} & \bf{97.0} & \bf{98.3} & \bf{0.31} \\
\hline 
\end{tabular}
\end{small}
\end{table*}

\subsection{Phoneme alignment}

Our next experiment focused on the phoneme alignment, which is used as a tool in developing speech recognition and text-to-speech systems. This is a structured prediction task --- the input $\vx$ represents a speech utterance, and consists of a pair $\vx=(\vs,\vp)$ of a sequence of acoustic feature vectors (mel-frequency cepstral coefficients) , $\vs=(\vs_1,\ldots,\vs_T)$, where $\vs_t \in \R^d$, $1 \le t \le T$; and a sequence of phonemes $\vp=(p_1,\ldots,p_K)$, where $p_k\in\Pc$, $1 \le k \le K$ is a phoneme symbol and $\Pc$ is a finite set of phoneme symbols. The lengths $K$ and $T$ can differ for different inputs, although typically $T$ is significantly larger than $K$. The goal is to generate an alignment between the two sequences in the input. The output $\vy$ is a sequence $(y_1,\ldots,y_K)$, where $1 \leq y_k \leq T$ is an integer giving the start frame in the acoustic sequence of the $k$-th phoneme in the phoneme sequence. Hence the $k$-th phoneme starts at frame $y_k$ and ends at frame $y_{k+1}\!-\!1$.

For this task we used the TIMIT speech corpus for which there are published benchmark results \citep{BrugnaraFaOm93,KeshetShSiCh07,Hosom09}. We divided a portion of the TIMIT corpus (excluding the SA1 and SA2 utterances) into three disjoint parts containing 1500, 1796 and 400 utterances, respectively. The first part was used to train a phoneme frame-based classifier, which given the pair of speech frame and a phoneme, returns the level of confidence that the phoneme was uttered in that frame. The output classifier is then used along with other features as a seven dimensional feature map $\vphi(\vx,\vy)=\vphi((\vs,\vp),\vy)$ as described in \cite{KeshetShSiCh07}.

The seven dimensional weight vector $\vw$ was trained on the second set of 150 aligned utterances for $\tau$-insensitive loss
\begin{equation} \label{eqn:tau_insensitive_loss}
\ell(\vy,\vyh) = \frac{1}{|\vy|} \max \left\{| \vy_k - \vyh_k | - \tau, 0\right\},
\end{equation}
with $\tau$ = 10 ms. This cost measures the average disagreement between all of the boundaries of the desired alignment sequence $\vy$ and the boundaries of predicted alignment sequence $\vyh$ where a disagreement of less than $\tau$ is ignored. 

We trained the system with the orbit update rule where $\eta=1.0/\sqrt{t}$ and $\lambda=0.2$; the structured perceptron update rule; the structural SVM optimized using stochastic gradient descent with $C$=5 \citep{shalev2011pegasos}; structured ramp-loss with $\eta=1.0/\sqrt{t}$, $\lambda=0.4$; and direct loss minimization algorithm with $\epsilon=1.1$ on a reduced training set of 150 examples (out of 1796) and a reduced validation set of 100 examples (out of 400). We were not able to train the system with the probit loss in a reasonable time.

The results are given in Table~\ref{tab:alignment}. The results in the first 4 columns should be read as the accuracy (in percentage) that the prediction was within $\tau$. The higher the better. The last column of the table is the actual loss computed by \eqref{eqn:tau_insensitive_loss} - the smaller the better. In those results the orbit update rule outperforms other algorithms, and yields state-of-the-art results. 

We would like to note that as in the MNIST experiment, the exponent values in the update rule were very close to 1 and, practically, the update rule \eqref{eq:orbit_update_rule2} could be used. 

\subsection{Vowel duration}

In the problem of vowel duration measurement we are provided with a speech signal which includes exactly one vowel preceded and followed by consonants (i.e., CVC). Our goal is to predict the vowel duration accurately. Precise measurement of vowel duration in a given context is needed in many phonological experiments, and currently is done manually \citep{heller2014grammatical}.

\begin{table*}
\renewcommand{\arraystretch}{1.2}
\caption{The cost function \eqref{eq:vowel_loss} divided into the onset and offset terms, as learned by different approaches that are aimed at minimizing the cost.}\label{tab:voweldur}
\vspace{0.15in}
\centering
\begin{footnotesize}
\begin{tabular}{lcccccccl}
\hline
{} & \multicolumn{3}{c}{Onset} & {} & \multicolumn{3}{c}{Offset} \\
\cline{2-4} \cline{6-8}
{} & $\tau_{b}=$0ms & $\tau_{b}=$10ms & $\tau_{b}=$20ms & {} & $\tau_{e}=$0ms & $\tau_{e}=$20ms & $\tau_{e}=$35ms & Run-time \\
\hline 
Perceptron & 1.394 & 1.040 & 0.757 & {} & 4.390 & 3.929 & 2.643 & 24m 18s \\
Probit & 5.215 & 5.671 & 6.603 & {} & 4.832 & 6.331 & 5.283 & 1d 12h 47m \\
Ramp loss & 6.507 & 6.501 & 5.46 & {} & 8.573 & 8.084 & 7.028 & 45m 36s \\
Direct loss & 1.426 & 1.068 & 0.933 & {} & {\bf 3.919} & {\bf 3.406} & 2.529 & 45m 21s \\
Orbit loss & {\bf 1.308} & {\bf 0.893} & {\bf 0.607} & {} & 4.178 & 3.707 & {\bf 2.531} & {\bf 24m 41s} \\
\hline 
\end{tabular}
\end{footnotesize}
\end{table*}
The speech signal is represented as a sequence of acoustic features $\vx$ = $(x_1, x_2,\dots,x_T)$ where each $x_i$ (1 $\leq$ i $\leq$ T) is a $\emph{d}$-dimensional vector representing acoustic parameters, such as high and low energy, pitch, voicing, correlation coefficient, and so on (we extract $d$=22 acoustic features every 5 msec). We denote the domain of the feature vectors by $\mathcal{X} \subset \mathbb{R}^\emph{d}$. The length of the input signal varies from one signal to another, thus $T$ is not fixed. We denote by $\mathcal{X}^*$ the set of all finite length sequences over $\mathcal{X}$. In addition, we denote by $t_{b}\in\mathcal{T}$ and $t_{e}\in\mathcal{T}$ the vowel onset and offset times, respectively, where $\mathcal{T} = \{1,...,T\}$. For brevity we set $\boldsymbol{t}=(t_{b}, t_{e})$. The typical duration of an utterance is around 2 sec. There were $n$=116 feature functions that described the typical duration of a vowel, the mean high energy before and after the vowel onset, and so on. The cost function we use is: 
\begin{equation}
\begin{split}
\label{eq:vowel_loss}
\ell (\hat{\vt},\vt) = \big[|\hat{t}_{b} - t_b| - \tau_{b}\big]_+ + \big[|\hat{t}_{e} - t_e| - \tau_{e}\big]_+,
\end{split}
\end{equation} 
where $[\pi]_+=\max\{0,\pi\}$, and $\tau_{b}$, $\tau_{e}$ are pre-defined constants. The above function measures the absolute differences between the predicted and the manually annotated vowel onsets and offsets. Since the manual annotations are not exact, we allow a mistake of $\tau_{b}$ and $\tau_{e}$ frames at the vowel onset and offset respectively.

We trained the system using the orbit update rule with $\eta=1/\sqrt{t}$, $\lambda=0.001$; the structured perceptron update rule; structured ramp loss with $\eta=0.1/\sqrt{t}$, $\lambda=0.8$; probit loss with the expectation approximated by a mean of 100 random samples $\eta=0.001/\sqrt{t}$, $\lambda=0.005$; and direct loss minimization with $\eta=0.1/\sqrt{t}$ and $\epsilon$=-1.52. All of those hyper-parameters were chosen for the validation set. The results are presented in Table~\ref{tab:voweldur} for different values of $\tau_{b}$ and $\tau_{e}$ in the cost function. It can be seen that the orbit is close to the direct loss minimization (differences of a frame or two on average) and is better than other approaches. Also note that as describe earlier, the efficiency of the orbit loss is similar to the structured perceptron update and better than other approaches.

%


\section{Discussion and Future Work}\label{sec:conclusion}

We introduced a new surrogate loss function that offers an efficient and effective learning rule. We gave a qualitative theoretical analysis presenting a PAC-Bayesian generalization bound and a consistency theorem. Despite the fact that the consistency property concerns the training performance when the number of training examples is big, the proposed loss function was shown to perform well on several tasks, even when the training set was of small or medium size. 

In terms of theoretical properties, we think that the theoretical analysis can be improved, and in particular we would like to have a better upper-bound of the probit loss in terms of the orbit loss, as expressed in Lemma~\ref{lem:l_probit_l_orbit}, which depends on the minimal distance between the predicted label and its closest neighbor label. Anyways, it is clear that when the norm of the weight vector becomes large relative to the norm of the noise, the inference with the weight vector and the inference with the perturbed weight vector -- both lead to the same predicted label with a high probability.

This work is part of our research on surrogate loss functions in the structured prediction setting. We believe that in order to understand what are good loss functions, we have to understand the interrelationship between them. While we showed some relation between the orbit loss, the Perceptron and the probit loss, we still think that more work should be done. We are especially interested in understanding the connection between the orbit, the probit, and the direct loss minimization approach.

\bibliographystyle{apalike}
\bibliography{machine_learning,speech_reco}
\end{document}